\documentclass{article}

\usepackage{microtype}
\usepackage{graphicx}
\usepackage{subfigure}
\usepackage{booktabs} 
\usepackage{multirow}

\usepackage{hyperref}

\usepackage[accepted]{icml2021}

\usepackage{amsmath,amsfonts,amsthm,bm}
\usepackage{bbm}

\newtheorem{theorem}{Theorem}
\newtheorem{lemma}{Lemma}

\newtheorem{assumption}{Assumption}

\def\eqref#1{equation~\ref{#1}}

\def\1{\mathbbm{1}}

\DeclareMathAlphabet{\mathsfit}{\encodingdefault}{\sfdefault}{m}{sl}
\SetMathAlphabet{\mathsfit}{bold}{\encodingdefault}{\sfdefault}{bx}{n}

\newcommand{\E}{\mathbb{E}}

\DeclareMathOperator*{\argmin}{arg\,min}

\newcommand{\X}{\mathcal{X}}

\newtheorem{corollary}{Corollary}
\newtheorem{definition}{Definition}

\newtheorem{remark}{Remark}

\theoremstyle{definition}

\usepackage[dvipsnames]{xcolor}

\usepackage{changepage}

\icmltitlerunning{Active Covering}

\begin{document}

\twocolumn[
\icmltitle{Active Covering}



\icmlsetsymbol{equal}{*}

\begin{icmlauthorlist}
\icmlauthor{Heinrich Jiang}{goo}
\icmlauthor{Afshin Rostamizadeh}{goo}
\end{icmlauthorlist}

\icmlaffiliation{goo}{Google Research}
\icmlcorrespondingauthor{Heinrich Jiang}{heinrichj@google.com}

\icmlkeywords{Active Learning}

\vskip 0.3in
]



\printAffiliationsAndNotice{} 

\begin{abstract}
We analyze the problem of active covering, where the learner is given an unlabeled dataset and can sequentially label query examples. The objective is to label query all of the positive examples in the fewest number of total label queries. We show under standard non-parametric assumptions that a classical support estimator can be repurposed as an offline algorithm attaining an excess query cost of $\widetilde{\Theta}(n^{D/(D+1)})$ compared to the optimal learner, where $n$ is the number of datapoints and $D$ is the dimension. We then provide a simple active learning method that attains an improved excess query cost of $\widetilde{O}(n^{(D-1)/D})$. Furthermore, the proposed algorithms only require access to the positive labeled examples, which in certain settings provides additional computational and privacy benefits. Finally, we show that the active learning method consistently outperforms offline methods as well as a variety of baselines on a wide range of benchmark image-based datasets.
\end{abstract}
\section{Introduction}

Active learning is an increasingly important practical area of machine learning as the amount of data grows faster than the resources to label these datapoints, which can be costly. In this paper, we introduce a variant of the active learning problem, called active covering, where the goal is to label query all of the positive examples given an unlabeled dataset in as few total queries as possible.

Active covering arises in many machine learning problems. In credit card fraud detection, one goal is to find the instances of fraud with as few queries as possible to the user asking if a transaction was fraudulent \cite{awoyemi2017credit}. In mineral exploration, the goal is to find all of the valuable resources with as little exploration as necessary \cite{acosta2019machine}. In computational drug discovery, one goal is to discover all of the effective drugs with as few suggested candidates as possible, as running trials on each candidate can be costly \cite{ou2012computational}. In bank loan applications, a label query means granting the loan to an applicant (as the only way to know if the applicant would default or not)-- as such, negative label queries can be costly \cite{khandani2010consumer}.   Finally, many internet applications need to moderate abusive content, and the goal is to quickly identify and remove all of the abusive content with as few label queries as possible to human operators \cite{nobata2016abusive}. 

For our analysis, we assume that an unlabeled pool of $n$ datapoints is generated by drawing i.i.d. samples from a mixture of two distributions, where one of the distributions represents the positive examples and the other the negative examples. Our goal is to retrieve all of the positive examples from this pool with as few label queries as possible.
Furthermore, we assume that the support of the positive examples is compact and its density function is lower bounded by a positive quantity. Finally, we leverage a few additional standard and mild nonparametric regularity assumptions on the the curvature of the support of positive examples. These assumptions allow for a wide range of distributions (e.g. mixtures of truncated multivariate Gaussians), and in particular the support of positives and negatives can be of any shape and can arbitrarily overlap.

We first establish the optimal active covering algorithm, which has knowledge of the support of positive examples. This algorithm, thus, will query all of the examples that lie in the support of positive examples (which will contain some negative examples where there is an overlap with support of negative examples). Then, the performance of any active covering algorithm can be compared to the optimal learner via the expected {\it excess query cost} that the algorithm incurs in addition to what the optimal learner is expected to incur. 

To begin with, we analyze an offline algorithm that is based on a classical method in support estimation, providing both upper and lower bounds on the excess query cost. We then provide an active learning procedure based on the explore-then-commit strategy, where we first explore with an initial random sample, and then exploit by querying the example closest to any of the positive examples found thus far. We show that the active learning procedure achieves a provably better excess query cost than the offline algorithm. 

The presented methods also have the beneficial property of using only the queried positive examples for learning. This not only has desirable computational implications in that we don't need to store the negative examples (especially in highly label imbalanced datasets with few positives), but also is practical in situations where there are strict privacy requirements for the negative examples. For example, this arises in the problem of fake account detection faced by social network platforms \citep{garcia2016discouraging}, where real account information is highly sensitive data and may even not be allowed to use for training. Another such application is spam email detection, where it may be desirable (or necessary) to avoid training with non-spam emails \citep{li2008privacy}. While this privacy aspect is not a focus of the paper, it may be a feature of independent interest.

We now summarize our contributions as follows.
\begin{itemize}
    \item In Section~\ref{sec:active_covering}, we introduce and formalize the active covering problem and establish the notion of excess query cost.
    \item In Section~\ref{sec:offline}, we analyze the offline learner and show matching upper and lower bounds on the excess query cost of $\widetilde{\Theta}(n^{D/(D+1)})$.
    \item In Section~\ref{sec:explore_commit}, we introduce and analyze the Explore-then-Commit active algorithm and show it has excess query cost $\widetilde{O}(n^{(D-1)/D})$.
    \item In Section~\ref{sec:experiments}, we show empirical results on a wide range of benchmark image-based datasets (Letters, MNIST, Fashion MNIST, CIFAR10, CelebA) comparing the Explore-then-Commit algorithm to a number of offline and active baselines.
\end{itemize}

\section{Active Covering}\label{sec:active_covering}
In this section, we formulate the active covering problem. We are given an unlabeled dataset of $n$ datapoints, $X$, and the goal is to minimize the number of label queries necessary until all positive examples are labeled. We provide the assumptions on the underlying distribution from which $X$ is drawn from and establish the notion of excess query cost, which is the additional label queries compared to the optimal procedure, which will be defined later.

\subsection{Theoretical Setup}

We make the following assumption on the data generating process, which says that with probability $p$, we draw a positive example from distribution $\mathcal{P}_+$ and with probability $1 - p$, we draw a negative example from distribution $\mathcal{P}_-$.
\begin{assumption}\label{a1}
The dataset $X$ is drawn i.i.d. from a distribution $\mathcal{P} := p\cdot \mathcal{P}_+ + (1-p)\cdot \mathcal{P}_-$, for some $p \in (0, 1)$ and 
$\mathcal{P}_+$ and $\mathcal{P}_-$ are distributions over $\mathbb{R}^D$ of the positive and negative examples respectively.
\end{assumption}
We then require the following regularity assumption on the distribution of positive examples. The first part ensures that the density of positive examples is lower bounded in its support; otherwise, some positive examples will only appear as outliers, making it impossible for any learner to find them in a non-trivial manner. The second part ensures that the support of positive examples does not become arbitrarily thin anywhere, otherwise there may not be any samples drawn from these regions and it will be very difficult to recover the entire support from a finite sample. This is a standard assumption in a variety of density estimation scenarios e.g. \cite{cuevas1997plug,singh2009adaptive}.
\begin{assumption}\label{a2}
 There exists density function $f_+ : \mathbb{R}^D \rightarrow \mathbb{R}$  corresponding to $\mathcal{P}_+$ with a compact support $\X_+$ that can be decomposed into a finite number of connected components. There exists $\lambda_0, r_0, C_+ > 0$ such that the following holds:
 \begin{itemize}

     \item $f_+(x) \ge \lambda_0$ for all $x \in \X_+$.
     
     \item For all $0 < r < r_0$ and $x\in \X_+$, we have $\text{Vol}(B(x, r) \cap X_+) \ge C_+ \cdot \text{Vol}(B(x, r))$, where $B(x, r) := \{x' \in \mathbb{R}^D : |x - x'| \le r\}$.
 \end{itemize}
\end{assumption}

The final assumption ensures that the negative example's density is upper bounded to ensure that there are no possible arbitrarily dense clumps of negative examples near the boundary of $\X_+$. Otherwise, querying in such regions may yield too few positive examples and it may be difficult to learn that such a region is part of the support of positive examples.
\begin{assumption}\label{a3}
There exists density function $f_- : \mathbb{R}^D \rightarrow \mathbb{R}$ corresponding to $\mathcal{P}_-$ and $\lambda_1 > 0$ such that $f_-(x) \le \lambda_1$ for all $x \in \X$.
\end{assumption}

Our assumptions are quite mild as they allow for a wide range of distributions. In particular, they are non-parametric so there are no assumptions that the data is generated based on a parameterized model. Moreover, the support of the positive examples (as well as the support of the negative examples) can be of arbitrary bounded shape, need not be a single connected component (i.e. can appear as a number of connected components), and can intersect arbitrarily with the support of the opposite label. Such mild assumptions can model the wide range of data distributions arising in practice. 

Perhaps the strongest of our assumptions is that the density on $\X_+$ is lower bounded. We stress that if the density of the positive examples on $\X_+$ can become arbitrarily small, then some regions can become so low-density that the positive examples in those regions can be considered outliers. For such outliers, it may be unrealistic to expect a procedure to efficiently find them. Nonetheless, in practice, our methodology can still be applied on datasets containing positive outliers, but without guarantees that those outliers will be efficiently recovered.
We made this assumption to keep our theoretical analysis from becoming too involved; however it's worth noting that it is possible to relax this assumption and allow the density to have smooth boundaries. To do so, we would assume parameters that bound the rate of decay of the density as it goes to $0$ and provide final guarantees that depend on the decay parameters. Assumptions used in recent analysis on densities with smooth boundaries \citep{zhao2020analysis} can be adapted here, which is a future research direction. 


\subsection{Optimal Learner and Excess Query Cost}

We will analyze a number of new learners for this problem. To do so, we establish a metric called {\it excess query cost}, which compares the learner to that of the the {\it optimal} learner, which takes the optimal strategy given knowledge of $\X_+$. This learner is unattainable in practice, but serves as a theoretical limit in which to quantify the excess query cost of an algorithm with respect to, to be defined below.

The optimal learner has knowledge of $\X_+$ and therefore its strategy is to label query every point in $\X_+$. Let $Q_{\text{opt}}$ be the number of label queries incurred by the optimal learner. 
Thus, the optimal learner attains an expected number of queries of:
\begin{align*}
    \E[Q_{\text{opt}}] = n \cdot \mathcal{P}(\X_+).
\end{align*}
We can then for any algorithm define the notion of excess query cost, which is the additional label queries needed compared to the optimal learner. That is, if the cost of an algorithm is $C$, then the excess query cost is
\begin{definition}[Excess Query Cost]
Suppose that an algorithm needs to make $Q$ label queries before labeling all of the positive examples. Then the excess query cost of the procedure is defined as: $C := Q - Q_{\text{opt}}$.
\end{definition}
For the results in the paper, we analyze the {\it expected} excess query cost, where the expectation is taken over the distribution from which sample $X$ is drawn from.

\subsection{Passive Learner}
We first provide a result for the passive learning algorithm, which queries labels uniformly at random until all positive examples have been retrieved and serves as our most basic baseline. In the following theorem we show a lower bound on the excess query cost that is linear in the pool size.
\begin{theorem}[Lower bound for Passive Learner]\label{theorem:passive}
There exists a distribution satisfying Assumptions~\ref{a1},~\ref{a2}, and~\ref{a3} such that with probability at least $\frac{1}{2}$, we have (letting $C_{\text{passive}}$ be the excess query cost of the passive learner):
    $\E[C_{\text{passive}}] \ge \frac{1}{2} \cdot n$.
\end{theorem}

\begin{table}[]
\centering
\begin{tabular}{lll}
 Algorithm & Excess Query Cost  \\
\hline
Passive & $\Theta(n)$   \\
Offline & $\widetilde{\Theta}(n^{D/(D+1)})$  \\
Active & $\widetilde{O}(n^{(D-1)/D})$  \\
\hline
\end{tabular}\caption{{\bf Summary of algorithms and results}. }
\end{table}

\section{Offline Learner}\label{sec:offline}

\begin{algorithm}[t]
   \caption{Offline Learner}
   \label{alg:offline}
\begin{algorithmic}
   \STATE {\bf Inputs}: Dataset $X$, initial sample size $m$.
   \STATE Let $X_0$ be $m$ examples sampled uniformly without replacement from $X$.
   \STATE Label query $X_0$ and let $X_{0,+}$ be the positive examples.
   \STATE Label query remaining examples in ascending order of minimum distance to $X_{0,+}$ (i.e. $x \rightarrow \min_{x' \in X_{0,+}} |x - x'|$) until all positive examples are label queried.
\end{algorithmic}
\end{algorithm}

The offline learner (Algorithm~\ref{alg:offline}) will first randomly sample $m$ points to label query and then label queries the remaining examples in ascending order of minimum distance to any of the initially sampled positive examples until all positives are labeled. It's worth noting that we may not know when all of the positive examples are labeled-- thus, in practice, we can terminate the algorithm when enough positives are found depending on the task or when the labeling budget runs out.

The offline learner is based on a classical technique in the support estimation literature \cite{cuevas1997plug}, where the support of a distribution can be covered with a finite sample of $m$ points drawn from this distribution by taking the $\epsilon$-neighborhood of the sample for appropriate $\epsilon$. We apply this methodology to only the positive examples in our initial sample of $m$ points. We show that the Algorithm~\ref{alg:offline} finishes when it label queries everything within $\epsilon \approx (\log(m) / m )^{1/D}$ of the initial positive examples, and thus it will cover all the examples in $\X_+$ and the excess cost will be proportional to $\epsilon \cdot n$ (all details in the Appendix).

The formal guarantee for the excess query cost is a follows:
\begin{theorem} [Excess Query Cost for Offline Learner]\label{theorem:offline} Suppose that Assumptions~\ref{a1},~\ref{a2}, and~\ref{a3} hold. Let $0 < \delta < 1$. There exists $C, M_0 > 0$ depending on $\mathcal{P}$ such that the following holds. In Algorithm~\ref{alg:offline}, suppose that $m$ is chosen sufficiently large such that $m \ge \frac{2\log(2/\delta)}{p^2}$ and $\frac{m}{\log m} \ge \log(4/\delta)\cdot M_0$.
Then, with probability at least $1-\delta$, Algorithm~\ref{alg:offline} has an expected excess query cost of:
\begin{align*}
    &\E[C_{\text{offline}}] \\
    &\le (1-p)\cdot m + C \cdot \left(\frac{\log(4/\delta) \cdot \log(p\cdot m/2)}{m}\right)^{1/D} \cdot n.
\end{align*}
\end{theorem}

We now have the following immediate result by optimizing $m$ as a function of $n$, trading off the cost from the initial exploration (first term) and the cost from the exploitation (second term):
\begin{corollary}Under the conditions of Theorem~\ref{theorem:offline}, setting $m \approx n^{D/(D+1)}$ results in expected excess query cost bounded as follows:
\begin{align*}
    \E[C_{\text{offline}}] \le \text{PolyLog}(n, 1/\delta) \cdot n^{D/(D+1)}.
\end{align*}
\end{corollary}

We also provide the following lower bound, which shows that the offline learner cannot achieve a better rate (proof found in the appendix).
\begin{theorem}[Lower Bound for Offline Learner]\label{theorem:offline_lower}
There exists a distribution satisfying Assumptions~\ref{a1}, ~\ref{a2}, and~\ref{a3} such that with probability at least $\frac{1}{4}$, we have for $n$ sufficiently large and some constant $C > 0$:
\begin{align*}
\E[C_{\text{offline}}] \ge (1-p)\cdot m + C\cdot \left(\frac{\log m}{m}\right)^{1/D}\cdot n.
\end{align*}
\end{theorem}

\section{Active Explore-then-Commit Learner}\label{sec:explore_commit}

We next show an active approach (Algorithm~\ref{alg:active_explore_commit}) inspired by Explore-then-Commit strategy \cite{garivier2016explore}  that proceeds by first exploring by randomly sampling a set of examples, and then commits to a greedy approach of choosing the closest unlabeled example to any positive example labeled thus far until all of the positive examples are labeled. 

\begin{algorithm}[t]
   \caption{Active Explore-then-Commit Learner}
   \label{alg:active_explore_commit}
\begin{algorithmic}
   \STATE {\bf Inputs}: Dataset $X$, initial sample size $m$.
   \STATE Let $X_0$ be $m$ examples sampled uniformly without replacement from $X$.
   \STATE Label query $X_0$ and let $X_{+, 0}$ be the positive examples.
   \STATE Initialize $X_p \leftarrow X_{+, 0}$ and $X_a \leftarrow X_0$
   \WHILE{not all positives examples in $X$ are labeled}
   \STATE Label query $x = \argmin_{x \in X \backslash X_a} d(x, X_p)$
   \IF{$x$ has a positive label}
   \STATE $X_p \leftarrow X_p \cup \{ x \}$.
   \ENDIF
   \STATE $X_a \leftarrow X_a \cup \{ x \}$.
   \ENDWHILE
\end{algorithmic}
\end{algorithm}

To analyze the algorithm there are three key steps: first we show that in the explore phase, we choose at least one example from each connected component (CC) of $\X_+$ (Lemma~\ref{lemma:one_from_each}). Next, we show that in any CC of $\X_+$, all of the positive examples are in the same connected component in the $\epsilon$-neighborhood graph for some $\epsilon$ specified later (Lemma~\ref{lemma:connectedness}). The final step is combining these two results  to show a bound on the excess query cost.

We now give the following result which says that for $m$ sufficiently large, depending on the probability mass distribution of the CCs of $\X_+$, we will with high probability have a positive example from each of the CCs in the initial sample.
\begin{lemma}\label{lemma:one_from_each}
Suppose Assumptions~\ref{a1} and~\ref{a2} hold and let $0 < \delta < 1$. Let the connected components of $\X_+$ be $\X_{+,1},...,\X_{+,c}$. Let $q := \min_{i \in [c]} \mathcal{P}_+(\X_{+, i})$. If
\begin{align*}
    m \ge \max\left\{ \frac{2\log(2c/\delta)}{p\cdot \log(1/(1-q))}, \frac{2\log(2/\delta)}{p^2} \right\},
\end{align*}
then with probability at least $1-\delta$, the initial $m$ examples will contain a positive example in each of $\X_{+,i}$ for $i \in [c]$.

\end{lemma}

The next result shows that the positive examples in each CC of $\X_+$ appear in the same CC of the $\epsilon$-neighborhood graph of the positive example for appropriate choice of $\epsilon$. This will be important in showing that after greedily sampling enough examples after the explore phase, we will query all of the  examples in $\X_+$ but not query examples that are more than $\epsilon$ away from $\X_+$.
\begin{lemma}[Connectedness]\label{lemma:connectedness} Suppose Assumptions~\ref{a1} and~\ref{a2} hold.
Let $0 < \delta < 1$ and $\X_{+,1},...,\X_{+,c}$ be the connected components of $\X_+$. The following holds with probability at least $1 - \delta$.
For each $i \in [c]$, we have that $\X_{+, i} \cap X_+$ is connected in the $\epsilon$-neighborhood graph of $X_+$, where
\begin{align*}
    \epsilon = 3\left(\frac{C_0\cdot D \log(2/\delta)\cdot \log n}{p\cdot C_+\cdot \lambda_0\cdot v_D \cdot n}\right)^{1/D},
\end{align*}
and $n$ is sufficiently large so that $\epsilon \le r_0$.
\end{lemma}

We now combine the two results to obtain an excess query cost guarantee for the active learner. Lemma~\ref{lemma:one_from_each} ensures that our initial sample of $m$ examples contains an example from each CC of $\X_+$ and Lemma~\ref{lemma:connectedness} ensures that when we actively sample in a greedy manner, we eventually query all of the positive examples and never query any example that is too far from $\X_+$-- this farness determines how much the active algorithm samples outside of $\X_+$ and hence determines the expected excess query cost.
\begin{theorem}[Excess query cost for Algorithm~\ref{alg:active_explore_commit}]\label{theorem:active_explore_commit}
Suppose Assumptions~\ref{a1},~\ref{a2}, and~\ref{a3} hold and let $0 < \delta < 1$. Let the connected components of $\X_+$ be $\X_{+,1},...,\X_{+,c}$ and $q := \min_{i \in [c]} \mathcal{P}_+(\X_{+, i})$. There exists constants $C, N_0 > 0$ depending on $\mathcal{P}$ such that the following holds. If
\begin{align*}
    m &\ge \max\left\{ \frac{2\log(4c/\delta)}{p\cdot \log(1/(1-q))}, \frac{2\log(4/\delta)}{p^2} \right\},
\end{align*}
and $\frac{n}{\log n} \ge N_0 \log(4/\delta)$,
then with probability at least $1-\delta$, we have the following excess query cost guarantee for Algorithm~\ref{alg:active_explore_commit}:
\begin{align*}
    \E[C_{\text{exp-commit}}] \le  m + C \cdot \left((\log(4/\delta)\cdot \log n\right)^{1/D} \cdot n^{(D-1)/D}.
\end{align*}
\end{theorem}

\begin{remark}
Our requirement for $m$ is tight w.r.t. $q$ and $c$ in the case where $q = \frac{1}{c}$ (i.e. equal probability of each CC). In this case, it reduces down to the classic coupon collector problem \cite{boneh1997coupon}: each CC is a coupon and the expected number of times we draw a coupon with replacement until we receive one example from each is $\Omega(c \log c) = \Omega(\log c / \log(1/(1-q)))$ by Taylor expansion of $\log(1/(1-q))$. 
\end{remark}

\begin{remark}
While our results all have a strong dependence on the dimension (commonly referred to as the curse of dimensionality), it's been shown that non-parametric techniques such as these algorithms can automatically adapt to the intrinsic dimension of the data and the convergence rates can be shown to depend only on this dimension and not the ambient dimension (i.e. arguments from \citet{pelletier2005kernel,kpotufe2011k,jiang2017density,jiang2019non} can be adapted here).
\end{remark}

\section{Related Works}

The problem of actively retrieving the positive labeled examples was studied as {\it active search} by \citet{garnett2012bayesian}, where the goal is to label query as many positive examples given a fixed budget and they propose a  sequential Bayesian approach that optimizes the expected number of positive labels across timesteps; however their method is computationally expensive requiring $O((2\cdot n)^\ell)$ runtime where $\ell$ is the number of lookahead timesteps to optimize over. Efficient approximations to this active search technique have been proposed \cite{jiang2018efficient,jiang2019cost}. In our theoretical setting, the goal is to label query all of the positive examples with as few label queries as possible rather than having a fixed known labeling budget. 

A recent work by \citet{jain2019new} designs an algorithm to actively identify the largest number of positive examples while minimizing or constraining the false-discovery rate (i.e. false-negative rate). 
They propose a bandit-style active elimination procedure which strategically label queries datapoints (i.e. each datapoint can be seen as an arm and label querying can be seen as pulling the arm) to find the best classification. In our work, we leverage the structure of the data while \citet{jain2019new} considers each datapoint as its own bandit arm and doesn't explicitly use information about the location of these datapoints. The contributions of \cite{jain2019new} are primarily in the theoretical analysis of the setting and proposed algorithm, while the practicality of the algorithm itself is limited due to its computation complexity.

A related line of work is learning under one-sided feedback, first studied under the name {\it apple tasting} by \citet{helmbold2000apple}, where the learner receives the true labels for only examples it predicted positively on and the goal is to have as high accuracy as possible. Recently, \citet{jiang2020learning} studied the one-sided feedback problem for generalized linear models and attain regret guarantees using an adaptive UCB-based approach under their proposed one-sided loss. This problem is similar to our proposed active covering problem in that both cases we desire to label query the positive examples; however, a key difference is that both \citet{helmbold2000apple} and \citet{jiang2020learning} operate in the {\it streaming} setting, where predictions must be made in real-time whereas here, the learner has access to the entire corpus of unlabeled data.

It's also worth mentioning the tangentially related set cover problem \cite{slavik1997tight}, where the goal is identify the smallest sub-collection of sets whose union equals the whole. The submodular set cover problem \cite{iwata2009submodular} involves finding a set that minimizes a modular cost function subject to a submodular function constraint. \citet{guillory2010interactive} propose an active approach to solve the submodular set cover problem. Active covering however is a different problem that tries to recover the set of datapoints rather than a collection of subsets.

Our work is also related to the support estimation literature, which has a long history. Some works include \citet{geffroy1964probleme,devroye1980detection,korostelev1993estimation,cuevas1997plug,biau2008exact}. Our offline algorithm applies the classical support estimator on the positive samples found to find a covering for $\X_+$, which is the union of the $\epsilon$-balls around the initial positive examples. Those works established both upper and lower bounds on $\epsilon$ of order $(\log m/m)^{1/D}$, which were key to the analysis for the offline algorithm.

More broadly, support estimation has also been studied under the name of one-class classification, where the goal is identify the examples of a particular class given only training examples from that class. There have been a wide range of approaches proposed including using SVMs \cite{scholkopf2001estimating,manevitz2001one}, density functions \cite{hempstalk2008one}, clustering \cite{ypma1998support}, optimization \cite{crammer2004needle}, and deep learning \cite{ruff2018deep}.


\begin{figure*}[!ht]
\includegraphics[width=\linewidth]{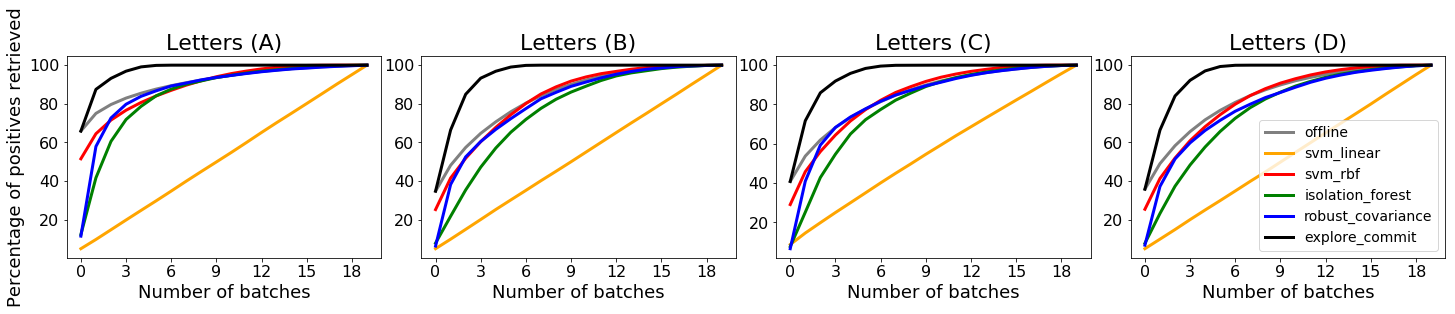}
\includegraphics[width=\linewidth]{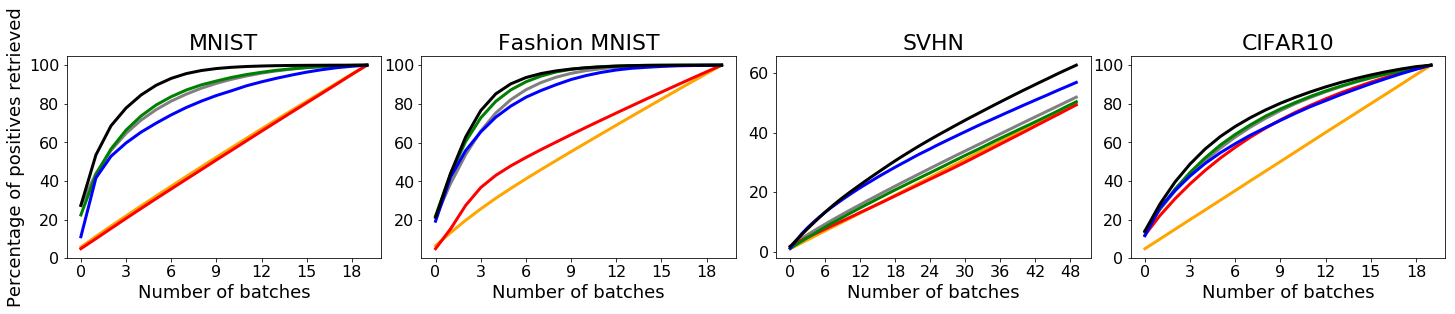}
\includegraphics[width=\linewidth]{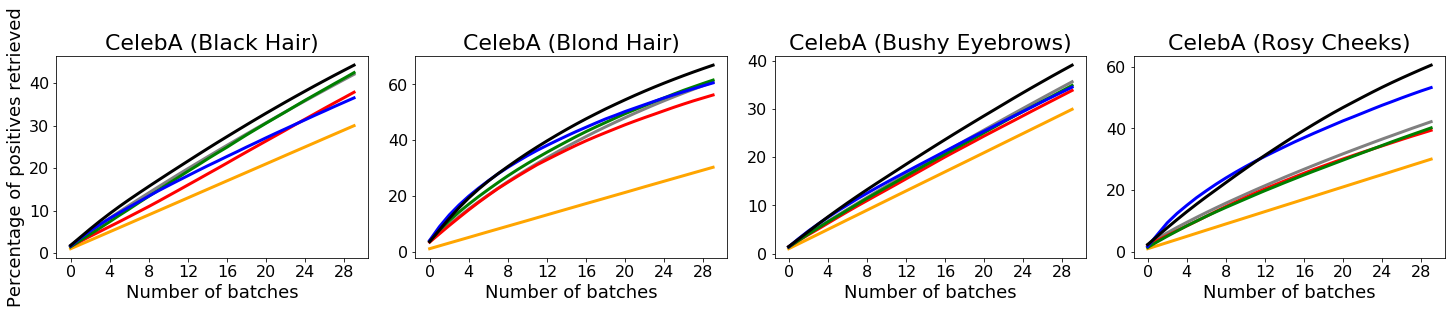}
\vspace{-0.7cm}
\caption{\label{fig:main_chart}{\bf Plots of percentage of all the positive examples retrieved after each batch}. {\bf Top}: Letters Recognition dataset for the first $4$ letters as the positive classes. {\bf Middle}: Various datasets using the label $4$ as the positive class. {\bf Bottom}: CelebA using various attributes as the label. We
compare our Explore-Commit method against the offline algorithm as well as the active variants of the baselines we tested across. We see that in all these cases, our method performs the best across batch sizes. Results averaged across $100$ runs.}
\end{figure*}

\section{Experiments} \label{sec:experiments}
In this section, we describe the details of our experimental results. We test the Explore-then-Commit algorithm (Algorithm~\ref{alg:active_explore_commit}) against a number of baselines including the offline algorithm (Algorithm~\ref{alg:offline}). We note that these algorithms do not come with any additional hyperparameters. 

\subsection{Baselines}
We propose a number of additional baselines based on the one-class classification methods implemented in scikit-learn \citep{pedregosa2011scikit} that can score examples based on likelihood that they are in-class, namely the One-Class SVM \citep{scholkopf2001estimating}, Isolation Forest \citep{liu2008isolation}, and Robust Covariance \cite{rousseeuw1999fast}.  For each of these baselines, we have variants: offline and active. The offline version trains the one-class classifier on the positive examples in the initial sample (see next subsection) and scores all of the unlabeled examples and then samples in order from most likely to least likely of being in-class. The active version retrains the one-class classifier on the label queried examples found thus far after each batch, and then scores the remaining examples to choose the next batch. All methods have the property of only utilizing the positive queried examples for learning.

We thus can list the baselines: {\bf 1.} Offline (O);
{\bf 2.} Offline Linear SVM (O-LS);
{\bf 3.} Active Linear SVM (A-LS);
{\bf 4.} Offline RBF SVM (O-RS);
{\bf 5.} Active RBF  SVM (A-RS);
{\bf 6.} Offline Isolation Forest (O-IF);
{\bf 7.} Active Isolation Forest  SVM (A-IF);
{\bf 8.} Offline Robust Covariance (O-RC);
{\bf 9.} Active Robust Covariance (A-RC).

\subsection{Experiment Setup}

For each of the datasets, we combine all of the data (i.e. any predetermined train/test/validation splits) into one dataset and operate on this dataset. We fix the initial sample size to a random stratified sample of $100$ datapoints. We train a neural network on the initial sample and use the activations of the second-last layer (i.e. the layer immediately before the logits) as an embedding for the data and we fix the embedding throughout and all of the methods will operate on this embedding instead of the original input. This is because recent work has shown that it's effective to use classical methods on the intermediate embeddings of the neural network \citep{papernot2018deep,bahri2020deep,bahri2021locally}; moreover, the original features may have undesirable properties (i.e. relative scaling of the features, high-dimensional pixel data, etc) which can hurt the performance of some of the baselines.

For each dataset, we run experiments using each of the classes as the positive class, as the datasets we used were all multiclass (with the exception of CelebA, which came with a wide range of binary attributes which we used as classes).
Then for each of the datasets with the exception of SVHN and CelebA, we let the batch size be $5\%$ of the remainder of the dataset (i.e. after removing the initial sample) to obtain $20$ batches, and for SVHN and CelebA, due to their size, we let the batch size be $1\%$ of the remainder of the dataset and ran for $50$ batches for SVHN and $30$ batches for CelebA. For all of the experimental results, we averaged across $100$ runs randomizing over different initial samples and ran on a cluster of NVIDIA\textsuperscript{TM} Tesla\textsuperscript{TM} V100 Tensor Core GPUs.

\subsection{Datasets and Embeddings}
We tested on the following datasets:\\
{\bf 1: UCI Letters Recognition} \cite{Dua:2019}, which has $20000$ datapoints and $16$ features based on various statistics on the pixel intensities of the original images of the letters, with $26$ classes -- one for each letter. To train the embedding, we used a fully-connected network with one hidden layer of $100$ units and ReLU activations and trained for $20$ epochs.\\
{\bf 2: MNIST}, with 70000 28x28 pixel grayscale images of handwritten  digits and $10$ classes. We use the same model and epochs as Letters for the embeddings. \\
{\bf 3: Fashion MNIST} \cite{xiao2017fashion}, with same dimensions and embedding training procedure as that of MNIST.\\
{\bf 4: CIFAR10} with 60000 32x32 colour images in 10 classes. For the embeddings use a simple VGG \cite{zhang2015accelerating} network and extract the second-last layer which has $128$ dimensions and train for $100$ epochs.\\
{\bf 5: SVHN} \citep{netzer2011reading} with 99289 color images, cropped to 32x32 pixels. For the embeddings, we use LeNet5 \cite{lecun1998gradient} and train for $20$ epochs. \\
{\bf 6. CelebA} \cite{liu2018large} a large-scale face attributes dataset with more than 162770; we resized the images to 28x28 celebrity images. The dataset has 40 attribute annotations which we use as separate binary classification tasks.  We use the same embedding procedure as that of SVHN.

\subsection{Hyperparameters}
Our method doesn't come with any additional hyperparameters; however the baselines do require the tuning of hyperparameters. For these methods, we perform $5$-fold cross-validation on the initial sample using accuracy as the metric (these methods as implemented in scikit-learn have predict methods which classifies whether an example is an outlier relative to the positive class). For the SVM methods, we tune the gamma parameter (kernel coefficient) and nu (upper bound on the fraction of training errors and a lower bound of the fraction of support vectors). For Isolation Forest, we tune the number of estimators in the ensemble. For Robust Covariance, we tune the proportion of contamination of the data set. For all of these aforementioned hyperparameters, we search over a grid of powers of two.

\subsection{Evaluation Metrics}
We plot the percentage of positive examples label queried across batches for each of the methods to illustrate the performance of each method. We also compute the area under the curve for each method, defined as the average number of positive examples retrieved across each batch, and use this as the primary evaluation metric to compare the methods. Since we average over $100$ runs, we also compute an error band on the area under the curve metric. We do this by computing the standard deviation of percentage of positive examples retrieved for each of the $20$ (or $50$ and $30$ in the case of SVHN and CelebA) batches. We average across these standard deviations and divide by square root of the number of runs to obtain an estimate of the standard deviation of the mean. We then form a $95\%$ confidence band based on this and consider methods which have overlapping bands as statistical ties.

\subsection{Results}
We show the results for each baseline and dataset/label pair under our area under the curve metric in Table~\ref{tab:auc}. Due to space, we could only show partial results for Letters and defer the rest along with the CelebA results to the Appendix. We nonetheless summarize all of the results here:\\
{\bf 1. Letters}. Our proposed method, Explore-then-Commit, outright outperforms all the other baselines on all $26$ tasks.\\
{\bf 2. MNIST}. Our method again outright outperforms all the other baselines on all $10$ tasks.\\
{\bf 3. Fashion MNIST}. Our method performs competitively on $9$ out of the $10$ tasks, with the next most competitive baseline (Active Isolation Forest) being competitive on $7$ out of the $10$ tasks.\\
{\bf 4. CIFAR10}. Here, our method performs competitively on $6$ of the $10$ tasks. It's worth noting that we only perform poorly when all of the methods perform poorly suggesting that in such settings not much learning is possible (i.e. from Table~\ref{tab:auc}, we only perform non-competitively when the AUC metric is under $55\%$. A passive learner that samples uniformly at random is expected to have an AUC of $50\%$).\\
{\bf 5. SVHN}. Our method is competitive for all tasks and outright wins for all but one task.\\
{\bf 6. CelebA}. Our method is competitive for 32 out of the 40 tasks. Due to space, the results are shown in the Appendix. We again see a similar pattern as in CIFAR10 where our method only performs poorly when all of the methods perform poorly (i.e. we only perform non-competitively when the AUC metric is under $20\%$. For comparison, a passive learner is expected to have an AUC of $15\%$).

\begin{table*}
\centering
\begin{tabular}{ |c|c|c|c|c|c|c|c|c|c|c|c|}
        \hline
   Dataset  &  Label & O & O-LS & A-LS & O-RS & A-RS & O-IF & A-IF & O-RC & A-RC & EC (Ours) \\
   \hline \hline
    \multirow{6}{*}{Letters}  & A & 91.14 & 52.52 & 52.49 & 84.81 & 89.02 & 59.52 & 84.69 & 64.27 & 86.87 & {\bf 97.12} \\ \cline{2-12}
 & B & 83.41 & 52.73 & 52.57 & 75.95 & 82.41 & 56.13 & 75.89 & 61.38 & 80.18 & {\bf 93.76} \\ \cline{2-12}
 & C & 84.48 & 56.19 & 56.08 & 75.92 & 83.78 & 55.78 & 78.68 & 59.21 & 81.92 & {\bf 94.2} \\ \cline{2-12}
 & D & 83.51 & 52.57 & 52.45 & 76.14 & 82.23 & 56.09 & 76.03 & 61.51 & 78.83 & {\bf 93.73} \\ \cline{2-12}
 & E & 78.6 & 52.85 & 52.99 & 74.84 & 81.41 & 55.77 & 73.7 & 60.17 & 77.27 & {\bf 89.5} \\ \cline{2-12}
 & F & 83.63 & 53.4 & 53.41 & 78.72 & 83.16 & 57.44 & 77.83 & 64.69 & 80.88 & {\bf 94.0} \\ \cline{2-12}
 & G & 82.23 & 52.72 & 52.67 & 75.76 & 81.88 & 58.15 & 74.58 & 63.45 & 79.79 & {\bf 92.35} \\  \hline\hline
  \multirow{10}{*}{MNIST}  & 0 & 86.67 & 81.81 & 81.96 & 52.53 & 52.95 & 83.44 & 90.8 & 74.31 & 86.48 & {\bf 94.44} \\ \cline{2-12}
 & 1 & 95.89 & 55.22 & 55.11 & 58.47 & 90.31 & 90.16 & 94.27 & 87.34 & 89.8 & {\bf 96.46} \\ \cline{2-12}
 & 2 & 75.86 & 60.64 & 60.37 & 52.54 & 52.56 & 72.66 & 80.18 & 62.81 & 77.83 & {\bf 85.47} \\ \cline{2-12}
 & 3 & 80.77 & 60.64 & 60.36 & 52.52 & 52.65 & 76.67 & 84.31 & 66.87 & 81.03 & {\bf 87.63} \\ \cline{2-12}
 & 4 & 83.05 & 54.23 & 54.14 & 52.47 & 53.08 & 76.86 & 83.61 & 66.71 & 78.31 & {\bf 89.14} \\ \cline{2-12}
 & 5 & 75.59 & 52.88 & 52.81 & 52.51 & 52.63 & 61.65 & 69.44 & 57.82 & 71.18 & {\bf 87.24} \\ \cline{2-12}
 & 6 & 86.53 & 59.98 & 59.77 & 52.49 & 54.03 & 81.19 & 88.33 & 67.37 & 81.95 & {\bf 93.24} \\ \cline{2-12}
 & 7 & 87.05 & 55.83 & 55.63 & 52.54 & 57.02 & 80.71 & 87.26 & 70.14 & 81.76 & {\bf 91.62} \\ \cline{2-12}
 & 8 & 75.7 & 56.27 & 56.17 & 52.49 & 52.62 & 69.73 & 78.3 & 61.71 & 77.32 & {\bf 83.37} \\ \cline{2-12}
 & 9 & 84.91 & 54.64 & 54.7 & 52.51 & 55.06 & 77.22 & 84.66 & 67.88 & 79.79 & {\bf 90.71} \\ \hline\hline
  \multirow{10}{*}{\shortstack[l]{Fashion\\ MNIST}}  & 0 & 87.81 & 54.51 & 54.49 & 52.9 & 66.76 & 86.35 & {\bf 90.14} & 81.5 & 87.44 & {\bf 89.75} \\ \cline{2-12}
 & 1 & 94.73 & 55.84 & 55.67 & 55.12 & 85.21 & 92.67 & 94.73 & 90.42 & 92.54 & {\bf 95.93} \\ \cline{2-12}
 & 2 & 84.44 & 55.57 & 55.57 & 52.72 & 63.65 & 82.97 & {\bf 87.6} & 78.46 & 85.6 & {\bf 87.19} \\ \cline{2-12}
 & 3 & 88.86 & 53.15 & 53.15 & 52.64 & 63.56 & 85.39 & 89.74 & 83.08 & 87.06 & {\bf 91.29} \\ \cline{2-12}
 & 4 & 84.9 & 56.47 & 56.41 & 52.68 & 62.54 & 83.23 & {\bf 87.25} & 79.68 & 83.61 & {\bf 88.09} \\ \cline{2-12}
 & 5 & 88.09 & 52.54 & 52.52 & 52.62 & 57.14 & 79.59 & 84.7 & 82.92 & 75.2 & {\bf 89.16} \\ \cline{2-12}
 & 6 & 77.31 & 52.5 & 52.5 & 52.98 & 63.62 & 75.94 & {\bf 81.63} & 71.46 & 80.18 & {\bf 81.1} \\ \cline{2-12}
 & 7 & 94.41 & 52.47 & 52.49 & 52.97 & 69.91 & 92.66 & {\bf 95.06} & 90.72 & 93.05 & {\bf 95.17} \\ \cline{2-12}
 & 8 & 82.86 & 55.43 & 55.46 & 52.5 & 54.1 & 78.23 & {\bf 86.56} & 78.33 & 83.97 & {\bf 85.96} \\ \cline{2-12}
 & 9 & 92.5 & 70.39 & 70.13 & 52.59 & 58.31 & 90.93 & {\bf 94.12} & 90.35 & 91.4 & 93.49 \\ \hline \hline 
\multirow{10}{*}{CIFAR10} & 0 & 67.06 & 54.33 & 54.27 & 64.84 & 64.8 & 64.6 & 68.03 & 65.49 & 69.36 & {\bf 70.69} \\ \cline{2-12}
 & 1 & {\bf 57.24} & 52.57 & 52.54 & {\bf 55.67} & 55.01 & 54.25 & 53.68 & {\bf 57.68} & 50.3 & 54.06 \\ \cline{2-12}
 & 2 & 65.43 & 52.48 & 52.51 & 59.75 & 59.92 & 62.09 & 64.61 & 63.42 & 65.15 & {\bf 68.71} \\ \cline{2-12}
 & 3 & {\bf 53.98} & 52.5 & 52.51 & {\bf 53.21} & {\bf 53.28} & {\bf 53.34} & 52.8 & {\bf 54.13} & 50.09 & 52.2 \\ \cline{2-12}
 & 4 & 70.94 & 52.5 & 52.55 & 66.68 & 67.52 & 68.54 & 71.52 & 67.95 & 68.29 & {\bf 73.97} \\ \cline{2-12}
 & 5 & {\bf 57.59} & 52.53 & 52.48 & 56.13 & 56.09 & 56.55 & 56.7 & {\bf 58.83} & 53.12 & 53.9 \\ \cline{2-12}
 & 6 & 72.16 & 52.48 & 52.49 & 67.67 & 67.89 & 67.79 & 70.87 & 70.78 & 65.26 & {\bf 74.73} \\ \cline{2-12}
 & 7 & {\bf 58.51} & 52.54 & 52.53 & 56.0 & 55.93 & 56.48 & {\bf 57.7} & {\bf 58.07} & 53.92 & {\bf 58.36} \\ \cline{2-12}
 & 8 & {\bf 70.25} & 52.48 & 52.47 & 66.67 & 66.86 & 67.92 & {\bf 71.57} & 68.13 & 69.91 & {\bf 71.42} \\ \cline{2-12}
 & 9 & 62.79 & 52.54 & 52.49 & 61.8 & 61.74 & 63.51 & {\bf 66.18} & 63.97 & 62.65 & 55.63 \\ 
 \hline\hline
 \multirow{10}{*}{SVHN} & 0 & 31.11 & 25.47 & 25.49 & 28.54 & 29.89 & 28.0 & 30.12 & 31.57 & {\bf 39.59} & {\bf 39.67} \\ \cline{2-12}
 & 1 & 28.23 & 25.53 & 25.52 & 25.63 & 25.25 & 25.08 & 25.44 & 32.81 & 35.19 & {\bf 37.32} \\ \cline{2-12}
 & 2 & 28.66 & 25.53 & 25.52 & 26.28 & 26.49 & 26.32 & 26.69 & 30.86 & 32.68 & {\bf 34.31} \\ \cline{2-12}
 & 3 & 28.19 & 25.57 & 25.56 & 26.11 & 26.47 & 26.34 & 26.79 & 29.37 & 31.0 & {\bf 33.81} \\ \cline{2-12}
 & 4 & 28.01 & 25.51 & 25.49 & 25.54 & 25.16 & 25.21 & 26.66 & 27.31 & 33.4 & {\bf 36.31} \\ \cline{2-12}
 & 5 & 29.02 & 25.56 & 25.53 & 26.55 & 26.98 & 26.77 & 27.75 & 29.67 & 32.81 & {\bf 34.44} \\ \cline{2-12}
 & 6 & 28.8 & 25.49 & 25.5 & 26.37 & 26.6 & 26.24 & 27.72 & 28.7 & 32.34 & {\bf 35.29} \\ \cline{2-12}
 & 7 & 29.36 & 25.52 & 25.48 & 26.46 & 26.18 & 25.7 & 27.22 & 27.79 & 35.14 & {\bf 37.62} \\ \cline{2-12}
 & 8 & 28.04 & 25.51 & 25.48 & 26.29 & 27.03 & 26.41 & 27.43 & 27.29 & 31.43 & {\bf 33.37} \\ \cline{2-12}
 & 9 & 28.48 & 25.49 & 25.49 & 26.82 & 27.5 & 26.28 & 28.04 & 27.62 & 32.83 & {\bf 34.36} \\ 
 \hline 
    \end{tabular}
    \caption{\label{tab:auc}{\bf Area under the curve metric for various benchmark image-based datasets}. For each of the datasets and possible labels, we show the area under the curve metric averaged across $100$ runs, with the top value bolded (any methods whose $95\%$ confidence intervals overlap were considered statistical ties). Due to space, we show the rest of the Letters results as well as the CelebA results in the Appendix.}
\end{table*}
\section{Conclusion}

We have formalized the problem of active covering and introduced several baselines, including a principled active approach that attains better guarantees than the offline algorithm. We showed in experiments that our proposed Explore-the-Commit algorithm has strong performance against a number of baselines while having desirable properties including not having additional hyperparameters, and not needing to store or use the queried negative examples. Future work involves extending theoretical analysis relaxing the hard boundary density assumption on $\X_+$, letting $\X_+$ be a lower dimensional manifold embedded in the $D$-dimensional space (rather than being full-dimensional) and attain excess query cost guarantees that depend on this lower dimension, and investigating the computational and privacy implications of such approaches.

{
\bibliography{main}
\bibliographystyle{icml2021}
}

\clearpage
{
\appendix
\onecolumn
{\Large \bf Appendix}
\section{Proofs}

\subsection{Proofs for Section~\ref{sec:active_covering}}

\begin{proof}[Proof of Theorem~\ref{theorem:passive}]
Consider the distribution where there are two connected components, one for $\X_+$ and one for $\X_-$, each with mixture probability $p = \frac{1}{2}$. Thus, Assumption~\ref{a1} holds and we are choose to free the other parameters of the distribution in any way that satisfies Assumption~\ref{a2} and \ref{a3} (e.g. a mixture of uniform density functions satisfies these assumptions). Now note that with probability $\frac{1}{2}$, the final point  that is label queried by the passive learner will be positive and, thus, the passive algorithm will need to query all of the points with probability $\frac{1}{2}$ in order to retrieve all positive points. In such an event, the excess query cost is at least $\frac{1}{2} \cdot n$.
\end{proof}

\subsection{Proofs for Section~\ref{sec:offline}}

Much of our technical results require the following uniform high-probability guarantee that balls of sufficient probability mass contain an example:
\begin{lemma}
\label{lemma:ball_bound}
Let $0 < \delta < 1$ and $\mathcal{F}$ be some distribution over $\mathbb{R}^D$ and $X$ be a sample of size $n$ drawn i.i.d. from $\mathcal{F}$. There exists universal constant $C_0$ such that the following holds with probability at least $1-\delta$ uniformly over all balls $B \in \mathbb{R}^D$:
\begin{align*}
    \mathcal{F}(B) \ge \frac{C_0\cdot D\cdot \log(2/\delta)\log n}{n} \Rightarrow |B \cap X| > 0.
\end{align*}
\end{lemma}
\begin{proof}
This follows by Lemma 7 of \citet{chaudhuri2010rates}.
\end{proof}

The following result 
bounds the volume of the $\epsilon$-neighborhood around $\X_+$, which will be used later to bound the excess number of points queried around $\X_+$. The result says that the volume of the $\epsilon$-neighboorhood around $\X_+$ (and not including $\X_+$) is linear in $\epsilon$. 
\begin{lemma}\label{lemma:volume}
Suppose Assumption~\ref{a2} holds. Then there exists constants $r_1, C_+' > 0$ depending only on $\X_+$ such that for all $0 < \epsilon < r_1$, we have
\begin{align*}
    \text{Vol}(B(\X_+, \epsilon) \backslash \X_+) \le C_+'\cdot \epsilon,
\end{align*}
where $B(\X_+, \epsilon) := \{x \in \mathbb{R}^D : \inf_{x' \in \X_+} |x-x'| \le \epsilon\}$.
\end{lemma}
\begin{proof}[Proof of Lemma~\ref{lemma:volume}]
This follows from \citet{gorin1983volume}. To see this, the equation on page 159 of \citet{gorin1983volume} states that if $M$ and $N$ are respectively  $d$-dimensional and $(d+k)$-dimensional compact smooth Riemannian manifolds and $f : M \rightarrow N$ is a smooth isometric embedding, then we have
\begin{align*}
    \text{Vol}(B(f(M), \epsilon)) = V_k \cdot \epsilon^k \cdot \text{Vol}(M) + O(\epsilon^{k+1}),
\end{align*}
where $V_k$ is the volume of a $k$-dimensional ball.
Here, we take $M = \mathcal{X}_+$ and $N = B(\mathcal{X}_+, r_1)$ for some $r_1 > 0$.
Then, we have $k = 0$ and taking $f$ to be the identity function, we have
\begin{align*}
    \text{Vol}(B(\X_+, \epsilon)) = \text{Vol}(\X_+) + O(\epsilon),
\end{align*}
and the result follows immediately.
\end{proof}

\begin{proof}[Proof of Theorem~\ref{theorem:offline}]

By Hoeffding's inequality, out of the initial $m$ examples that Algorithm~\ref{alg:offline} label queries, we have with probability at least $1 - \delta/2$ that at least $p - \sqrt{\frac{1}{2m}\cdot  \log(2/\delta)}$ fraction of them are positively labeled, since the example being positive follows a Bernoulli distribution with probability $p$. Then by the condition on $m$, we have that at least $p/2$ fraction of the points are positively labeled.

Take \begin{align*}
\epsilon = \left(\frac{2\cdot C_0\cdot D\cdot \log(4/\delta)\log(p\cdot m/2)}{p^2 \cdot \lambda_0\cdot C_+\cdot v_D \cdot m}\right)^{1/D}, \hspace{0.5cm} M_0 =\max\left\{ \frac{2\cdot C_0\cdot D(\log(p\cdot /2) + 1)}{p^2 \cdot \lambda_0\cdot C_+\cdot v_D \cdot \min\{r_0,r_1\}^D},  2e\right\},
\end{align*}
where $v_D$ is the volume of a unit ball in $\mathbb{R}^D$. Then, we have that the condition on $m$ and $M_0$ guarantees that $\epsilon < \min\{r_0, r_1\}$.

Let $x \in \X_+$. We have that the probability mass of positive examples in $B(x, \epsilon)$ w.r.t. $\mathcal{P}$ is:
\begin{align*}
    p\cdot \mathcal{P}_+(B(x, \epsilon)) &\ge p\cdot \lambda_0 \cdot \text{Vol}(B(x, \epsilon) \cap \X_+)\\
    &\ge p\cdot \lambda_0 \cdot C_+ \cdot \text{Vol}(B(x, \epsilon))\\
    &\ge p\cdot \lambda_0 \cdot C_+ \cdot v_D \cdot \epsilon^D \\
    &\ge  \frac{2\cdot C_0\cdot D\log(4/\delta)\log(p\cdot m/2)}{p\cdot m}. 
\end{align*}
Then by Lemma~\ref{lemma:ball_bound}, we have with probability at least $1-\delta/2$ that all the positive examples in $X$ are within $\epsilon$ of one of the positive examples among the initially sampled $m$ examples. Therefore, Algorithm~\ref{alg:offline} retrieves all of the positive examples.

Now we bound the expected regret:
\begin{align*}
    \E[C_{\text{offline}}] &\le (1-p)\cdot m + n\cdot (1-p) \cdot \mathcal{P}_-(B(\X_+, \epsilon)\backslash \X_+) \\
   & \le (1-p)\cdot m + n\cdot (1-p) \cdot \lambda_1\cdot C_+'\cdot \epsilon.
\end{align*}
The result follows.
\end{proof}

\begin{proof}[Proof of Theorem~\ref{theorem:offline_lower}]
Let $\mathcal{P}_+$ be the uniform distribution on the unit hypercube $[0, 1]^D$ and $\mathcal{P}_-$ be the uniform distribution on $[-1, 2]^D$. In the initial sampling phase of Algorithm~\ref{alg:offline}, at most $m$ of the examples will be positively labeled. Let $\widehat{\X_+} = X \cap \left(\cup_{x\in X_{0,+}} B(x, \epsilon) \right)$, the set of points that Algorithm~\ref{alg:offline} labeled.
Then, Theorem 3b in \cite{cuevas1997plug} shows that for $n$ sufficiently large, with probability at least $1/4$, we have
\begin{align*}
    d_H(\widehat{\X_+}, \X_+) \ge \frac{1}{4}\left(\frac{\log m}{m}\right)^{1/D}
\end{align*}
for any $\epsilon > 0$, where $d_H(A, B) := \max \{\sup_{x \in A} d(x, B), \sup_{x \in B} d(x, A) \}$ is the Hausdorff distance. Therefore, we have (in the case of taking $\epsilon \rightarrow 0$):
\begin{align*}
    d_H(X_{0, +}, \X_+) \ge \frac{1}{4}\left(\frac{\log m}{m}\right)^{1/D}.
\end{align*}
Since $X_{0, +} \subseteq \X_+$, it follows that $d_H(X_{0, +}, \X_+) = \sup_{x \in \X_+} d(x, X_{0, +})$. Therefore, we need $\epsilon \ge \frac{1}{4}\left(\frac{\log m}{m}\right)^{1/D}$ in order for Algorithm~\ref{alg:offline} to recover all of the positive examples. Thus, the expected regret is at least (for some $C > 0$)
\begin{align*}
    \E[C_{\text{offline}}] \ge (1-p)\cdot m + C\cdot \left(\frac{\log m}{m}\right)^{1/D} \cdot n,
\end{align*}
as desired.
\end{proof}

\subsection{Proofs for Section~\ref{sec:explore_commit}}

\begin{proof}[Proof of Lemma~\ref{lemma:one_from_each}]
By Hoeffding's inequality, out of the initial $m$ examples that Algorithm~\ref{alg:active_explore_commit} label queries, we have with probability at least $1 - \delta/2$ that at least $p - \sqrt{\frac{1}{2m}\cdot  \log(2/\delta)}$ are fraction of them are positively labeled, since the example being positive follows a Bernoulli distribution with probability $p$. Then by the condition on $m$, we have that at least $p/2$ fraction of the points are positively labeled and thus we have at least $mp/2$ positive examples.

Then, we have that out of these $mp/2$ examples, the probability that none of them are in $\X_{+, i}$ for each $i \in [c]$ is at most
\begin{align*}
    (1 - \mathcal{P}_+(\X_{+, i}))^{mp/2} \le (1 - q)^{mp/2} \le \frac{\delta}{2c}.
\end{align*}
The result follows by union bound.
\end{proof}

\begin{proof}[Proof of Lemma~\ref{lemma:connectedness}]

Let $x, x' \in \X_{+, i}$. There exists a path $x=x_1 \to x_2 \to \ldots \to x_q=x'$ in $\X_{+, i}$ such that $||x_j - x_{j+1}|| \le \epsilon/3$.
We also have that the probability mass of positive examples in $B(x_j, \epsilon/3)$ w.r.t. $\mathcal{P}$ is:
\begin{align*}
    p\cdot \mathcal{P}_+(B(x_j, \epsilon/2)) &\ge p\cdot C_+ \cdot \lambda_0 \cdot v_D \cdot (\epsilon/3)^D \ge \frac{C_0\cdot D \log(2/\delta)\cdot \log n}{n}.
\end{align*}
Therefore by Lemma~\ref{lemma:ball_bound}, there exists $x_j' \in B(x_j, \epsilon/3) \cap X_+$, where $X_+$ are the positive examples in $X$.
Hence, by triangle inequality, there exists path $x=x_1' \to x_2' \to ... \to x_q'=x'$ all in $X_+$ where $||x_j' - x_{j+1}'|| \le ||x_j' - x_j|| + ||x_j - x_{j+1}|| + ||x_{j+1} - x_{j+1}'|| \le \epsilon$ implying that $\X_{+, i} \cap X_+$ is connected in the $\epsilon$-neighborhood graph of $X_+$. The result follows immediately.
\end{proof}

Finally, we combine these two results to show the final excess query cost guarantee for Algorithm~\ref{alg:active_explore_commit}.

\begin{proof}[Proof of Theorem~\ref{theorem:active_explore_commit}]
Take 
\begin{align*}
N_0 = \frac{3^D\cdot C_0\cdot D }{\min\{ r_0, r_1\}^D\cdot p \cdot C_+\cdot \lambda_0\cdot v_D}.
\end{align*}
By Lemma~\ref{lemma:one_from_each}, there exists at least one positive example in the initial $m$ samples from each connected component of $\X_+$.
Define
\begin{align*}
    \epsilon := 3\left(\frac{C_0\cdot D\log(4/\delta)\cdot \log n}{p\cdot C_+\cdot \lambda_0\cdot v_D \cdot n}\right)^{1/D}.
\end{align*}
We have that the condition on $n$ implies that $\epsilon \le \min\{r_0, r_1\}$. By Lemma~\ref{lemma:connectedness}, we have that all of the positive examples of each connected component of $\X_+$ are in the same CC of the $\epsilon$-neighborhood graph of the positive examples.
Therefore, when the algorithm terminates, the set of examples it will select from will be contained in $B(\X_+, \epsilon)$.

Therefore, we have 
\begin{align*}
    \E[C_{\text{exp-commit}}] \le (1-p)\cdot m + C_+' \cdot \epsilon \cdot n
    \le (1-p)\cdot m + C \cdot \left((\log(4/\delta)\cdot \log n\right)^{1/D} \cdot n^{(D-1)/D},
\end{align*}
for some $C$ depending on $\mathcal{P}$, as desired.
\end{proof}

\section{Additional Experiment Plots}
In Table~\ref{table:letters}, we show the full results for the Letters dataset for the area under curve metrics. We see that in all cases, our method outperforms outright. In Table~\ref{table:celeba}, we show the full results for CelebA. We see that our method is competitive for 32 out of the 40 tasks.

\begin{table*}
\centering
\begin{tabular}{ |c|c|c|c|c|c|c|c|c|c|c|c|}
        \hline
   Dataset  &  Label & O & O-LS & A-LS & O-RS & A-RS & O-IF & A-IF & O-RC & A-RC & EC (Ours) \\
   \hline 
   \multirow{26}{*}{Letters}  & A & 91.14 & 52.52 & 52.49 & 84.81 & 89.02 & 59.52 & 84.69 & 64.27 & 86.87 & {\bf 97.12} \\ \cline{2-12}
 & B & 83.41 & 52.73 & 52.57 & 75.95 & 82.41 & 56.13 & 75.89 & 61.38 & 80.18 & {\bf 93.76} \\ \cline{2-12}
 & C & 84.48 & 56.19 & 56.08 & 75.92 & 83.78 & 55.78 & 78.68 & 59.21 & 81.92 & {\bf 94.2} \\ \cline{2-12}
 & D & 83.51 & 52.57 & 52.45 & 76.14 & 82.23 & 56.09 & 76.03 & 61.51 & 78.83 & {\bf 93.73} \\ \cline{2-12}
 & E & 78.6 & 52.85 & 52.99 & 74.84 & 81.41 & 55.77 & 73.7 & 60.17 & 77.27 & {\bf 89.5} \\ \cline{2-12}
 & F & 83.63 & 53.4 & 53.41 & 78.72 & 83.16 & 57.44 & 77.83 & 64.69 & 80.88 & {\bf 94.0} \\ \cline{2-12}
 & G & 82.23 & 52.72 & 52.67 & 75.76 & 81.88 & 58.15 & 74.58 & 63.45 & 79.79 & {\bf 92.35} \\ \cline{2-12}
 & H & 69.07 & 52.4 & 52.63 & 61.19 & 69.51 & 53.95 & 64.15 & 52.03 & 66.94 & {\bf 81.78} \\ \cline{2-12}
 & I & 84.26 & 52.37 & 52.59 & 73.96 & 80.72 & 55.95 & 77.56 & 59.11 & 83.83 & {\bf 93.89} \\ \cline{2-12}
 & J & 84.24 & 54.45 & 54.35 & 75.16 & 81.72 & 56.22 & 77.06 & 58.58 & 81.61 & {\bf 94.28} \\ \cline{2-12}
 & K & 72.7 & 52.57 & 52.62 & 65.15 & 72.03 & 55.46 & 68.12 & 51.45 & 74.38 & {\bf 89.4} \\ \cline{2-12}
 & L & 80.98 & 52.59 & 52.67 & 72.74 & 79.87 & 55.44 & 77.73 & 59.16 & 78.57 & {\bf 93.18} \\ \cline{2-12}
 & M & 81.83 & 54.41 & 54.25 & 77.48 & 83.23 & 59.55 & 78.61 & 60.62 & 81.67 & {\bf 93.07} \\ \cline{2-12}
 & N & 75.15 & 52.67 & 52.68 & 68.58 & 75.31 & 55.37 & 69.43 & 55.85 & 73.74 & {\bf 89.8} \\ \cline{2-12}
 & O & 87.71 & 52.51 & 52.62 & 80.88 & 86.08 & 57.88 & 78.08 & 67.62 & 82.71 & {\bf 94.69} \\ \cline{2-12}
 & P & 84.76 & 52.7 & 52.7 & 79.24 & 84.47 & 57.63 & 79.18 & 64.81 & 83.19 & {\bf 93.47} \\ \cline{2-12}
 & Q & 82.25 & 53.22 & 53.08 & 74.97 & 80.09 & 55.71 & 74.66 & 60.49 & 79.55 & {\bf 92.18} \\ \cline{2-12}
 & R & 82.68 & 52.59 & 52.48 & 76.72 & 82.47 & 56.32 & 75.3 & 61.37 & 79.26 & {\bf 92.77} \\ \cline{2-12}
 & S & 76.51 & 52.83 & 52.91 & 70.42 & 75.85 & 55.24 & 72.01 & 59.66 & 75.84 & {\bf 89.27} \\ \cline{2-12}
 & T & 83.47 & 55.32 & 55.35 & 76.38 & 82.95 & 56.87 & 79.06 & 62.87 & 81.96 & {\bf 93.34} \\ \cline{2-12}
 & U & 78.05 & 52.91 & 53.05 & 73.72 & 81.28 & 55.86 & 72.81 & 58.56 & 76.41 & {\bf 92.7} \\ \cline{2-12}
 & V & 89.82 & 54.61 & 54.59 & 82.3 & 87.65 & 59.74 & 82.18 & 61.01 & 85.18 & {\bf 96.19} \\ \cline{2-12}
 & W & 90.15 & 55.25 & 55.29 & 84.57 & 89.11 & 59.37 & 82.04 & 63.1 & 86.65 & {\bf 96.79} \\ \cline{2-12}
 & X & 80.18 & 52.63 & 52.6 & 74.84 & 80.68 & 55.92 & 73.04 & 62.48 & 78.1 & {\bf 92.93} \\ \cline{2-12}
 & Y & 81.73 & 54.65 & 54.7 & 72.3 & 79.67 & 57.49 & 76.38 & 53.48 & 79.33 & {\bf 92.46} \\ \cline{2-12}
 & Z & 83.78 & 54.6 & 54.54 & 76.89 & 84.18 & 56.46 & 78.56 & 60.14 & 82.64 & {\bf 93.35} \\ 

        \hline
    \end{tabular}
    \caption{\label{table:letters}{\bf Letters}: Area under curve metric.}
\end{table*}

\begin{table*}
\centering
\begin{tabular}{ |c|c|c|c|c|c|c|c|c|c|c|c|}
        \hline
    Label & O & O-LS & A-LS & O-RS & A-RS & O-IF & A-IF & O-RC & A-RC & EC (Ours) \\
   \hline 
  5-o-Clock-Shadow & 20.17 & 15.54 & 15.55 & 17.31 & 18.28 & 17.89 & 18.85 & 20.25 & 22.49 & {\bf 23.85} \\ \hline
Arched-Eyebrows & 19.29 & 15.52 & 15.51 & 17.54 & 17.69 & 18.34 & 18.65 & 19.87 & 19.76 & {\bf 20.93} \\ \hline
Attractive & 18.4 & {\bf 19.19} & 18.32 & 18.6 & 18.7 & {\bf 18.86} & {\bf 18.89} & 18.03 & 17.93 & 18.61 \\ \hline
Bags-Under-Eyes & 17.06 & 15.5 & 15.51 & 16.58 & {\bf 17.54} & {\bf 17.26} & 17.11 & {\bf 17.3} & 16.33 & 16.97 \\ \hline
Bangs & {\bf 20.88} & 15.52 & 15.52 & {\bf 20.6} & 19.97 & 20.49 & {\bf 21.43} & {\bf 22.12} & 19.9 & {\bf 22.08} \\ \hline
Bald & 18.52 & 15.48 & 15.49 & 16.6 & 16.98 & 15.48 & 17.92 & NA & NA & {\bf 28.34} \\ \hline
Big-Lips & 15.81 & 15.51 & 15.5 & 15.44 & 14.91 & 15.26 & 15.46 & 15.61 & 15.68 & {\bf 16.35} \\ \hline
Big-Nose & 16.58 & 15.53 & 15.53 & 15.92 & 16.58 & 15.95 & 16.23 & 16.88 & 16.12 & {\bf 17.23} \\ \hline
Black-Hair & 22.87 & 15.49 & 15.5 & 20.8 & 19.43 & 21.35 & 22.52 & 21.23 & 20.5 & {\bf 24.6} \\ \hline
Blond-Hair & 36.36 & 15.67 & 15.67 & 31.41 & 34.64 & 35.92 & 37.84 & 37.31 & 39.24 & {\bf 41.72} \\ \hline
Blurry & 16.75 & 15.44 & 15.48 & 16.38 & 16.06 & 16.48 & 16.66 & 16.66 & 16.81 & {\bf 17.79} \\ \hline
Brown-Hair & 20.88 & 15.48 & 15.48 & 18.82 & 20.45 & 20.64 & {\bf 21.34} & 20.94 & 20.77 & {\bf 21.57} \\ \hline
Bushy-Eyebrows & 19.11 & 15.49 & 15.49 & 17.87 & 18.12 & 18.18 & 18.7 & 18.88 & 19.2 & {\bf 21.3} \\ \hline
Chubby & 16.77 & 15.51 & 15.5 & 15.98 & 16.2 & 15.83 & 16.15 & 16.51 & {\bf 19.01} & {\bf 18.63} \\ \hline
Double-Chin & 18.05 & 15.53 & 15.52 & 16.73 & 17.31 & 16.17 & 17.28 & 16.66 & {\bf 22.57} & {\bf 22.3} \\ \hline
Eyeglasses & 16.48 & 15.51 & 15.5 & 16.14 & {\bf 17.12} & {\bf 16.48} & {\bf 16.78} & {\bf 17.74} & 15.88 & 15.44 \\ \hline
Goatee & 17.57 & 15.49 & 15.51 & 16.86 & 16.69 & 16.22 & 16.93 & 16.98 & 17.98 & {\bf 19.81} \\ \hline
Gray-Hair & 23.19 & 15.53 & 15.55 & 19.55 & 21.77 & 17.01 & 23.12 & 20.84 & {\bf 32.86} & {\bf 31.66} \\ \hline
Heavy-Makeup & 21.42 & 15.85 & 15.58 & 20.49 & 20.5 & 21.24 & {\bf 21.94} & 21.29 & 20.4 & {\bf 22.45} \\ \hline
High-Cheekbones & 19.17 & 15.38 & 15.17 & 18.59 & 18.79 & 18.91 & 18.89 & 19.78 & 18.77 & {\bf 20.07} \\ \hline
Male & 18.64 & 15.49 & 15.57 & {\bf 20.09} & 19.36 & 19.14 & 18.8 & 18.43 & 16.5 & 19.24 \\ \hline
Mouth-Slightly-Open & 17.37 & 15.66 & 15.56 & 17.43 & 17.34 & 17.42 & 17.16 & {\bf 17.79} & 17.2 & {\bf 17.77} \\ \hline
Mustache & 17.16 & 15.5 & 15.48 & 16.77 & 16.46 & 15.98 & 16.99 & 16.64 & 17.05 & {\bf 18.13} \\ \hline
Narrow-Eyes & 15.48 & 15.49 & 15.5 & 15.66 & {\bf 15.95} & {\bf 15.78} & {\bf 15.78} & 15.68 & 15.36 & 14.83 \\ \hline
No-Beard & 15.82 & {\bf 16.38} & {\bf 16.37} & 15.79 & 15.93 & 16.0 & 15.99 & 16.08 & 15.98 & 15.83 \\ \hline
Oval-Face & 18.02 & 15.5 & 15.5 & 16.76 & 17.11 & 17.18 & 17.37 & 18.56 & 18.37 & {\bf 19.22} \\ \hline
Pale-Skin & 18.15 & 15.45 & 15.48 & {\bf 19.89} & {\bf 19.15} & 16.95 & {\bf 20.95} & {\bf 18.75} & 17.94 & {\bf 19.55} \\ \hline
Pointy-Nose & 18.55 & 15.49 & 15.49 & 17.03 & 17.22 & 17.23 & 17.75 & 18.69 & 18.8 & {\bf 19.97} \\ \hline
Receding-Hairline & 18.5 & 15.5 & 15.5 & 17.09 & 17.47 & 16.67 & 17.55 & 18.61 & {\bf 22.2} & {\bf 22.38} \\ \hline
Rosy-Cheeks & 23.96 & 15.5 & 15.51 & 19.77 & 22.49 & 18.5 & 22.3 & 22.45 & 32.64 & {\bf 34.69} \\ \hline
Sideburns & 18.32 & 15.51 & 15.52 & 17.23 & 17.19 & 16.9 & 17.39 & 17.41 & 19.36 & {\bf 21.77} \\ \hline
Smiling & 19.46 & 16.01 & 15.52 & 18.97 & 19.05 & 19.31 & 19.15 & 19.89 & 18.81 & {\bf 20.15} \\ \hline
Straight-Hair & {\bf 16.1} & 15.5 & 15.5 & {\bf 15.97} & {\bf 16.25} & 15.91 & {\bf 16.02} & {\bf 16.24} & {\bf 16.13} & {\bf 16.29} \\ \hline
Wavy-Hair & 21.0 & 15.52 & 15.52 & 20.18 & 20.09 & 20.56 & 20.97 & 21.13 & 20.42 & {\bf 21.88} \\ \hline
Wearing-Earrings & 18.75 & 15.47 & 15.47 & 16.93 & 17.87 & 17.75 & 18.37 & 19.28 & 20.0 & {\bf 20.62} \\ \hline
Wearing-Hat & {\bf 18.32} & 15.48 & 15.49 & {\bf 18.33} & {\bf 18.27} & 17.44 & {\bf 19.95} & {\bf 20.12} & 15.16 & 16.82 \\ \hline
Wearing-Lipstick & 20.42 & 19.0 & 17.34 & 19.64 & 19.65 & 20.49 & {\bf 20.8} & 20.21 & 19.51 & {\bf 21.06} \\ \hline
Wearing-Necklace & 18.99 & 15.48 & 15.49 & 16.72 & 17.85 & 17.54 & 18.15 & 19.44 & {\bf 21.28} & {\bf 21.48} \\ \hline
Wearing-Necktie & 19.56 & 15.52 & 15.51 & 17.51 & 18.13 & 16.53 & 18.36 & 18.45 & {\bf 23.79} & {\bf 24.46} \\ \hline
Young & 15.51 & {\bf 16.22} & {\bf 16.21} & 15.54 & 15.63 & 15.55 & 15.55 & 15.56 & 15.51 & 15.57 \\ 
        \hline
    \end{tabular}
    \caption{\label{table:celeba}{\bf CelebA}: Area under curve metric. We note that for Bald, there were no results for the Robust Covariance metrics. This is because due to the low rate of positive examples, it was not possible to tune Robust Covariance's hyper-parameters via cross-validation on the initial sample.}
\end{table*}

}

\end{document}